\documentclass[letterpaper, 10 pt, conference]{ieeeconf}

\IEEEoverridecommandlockouts

\makeatletter
\let\NAT@parse\undefined
\makeatother

\usepackage{graphicx}
\usepackage[caption=false,font=footnotesize]{subfig}

\usepackage{amsmath,amsthm,amssymb}
\usepackage{mathtools}
\mathtoolsset{showonlyrefs}
\usepackage{comment}
\usepackage{pgfplots}
\usepackage{tikz}
\pgfplotsset{compat = newest}

\def\tr{^{\rm T}}

\def\de{\mathrm{d}}

\newcommand{\mc}{\mathcal}

\newtheorem{theorem}{Theorem}

\newtheorem{proposition}[theorem]{Proposition}

\newtheorem{definition}[theorem]{Definition}

\newtheorem{remark}[theorem]{Remark}
\newtheorem{example}[theorem]{Example}

\DeclareMathOperator*{\argmin}{arg\,min~}
\DeclareMathOperator*{\minimize}{minimize~}
\DeclareMathOperator*{\st}{subject\,to~}

\newcommand{\R}{\mathbb{R}}

\title{\LARGE \bf
A Set-Theoretic Approach to Multi-Task Execution and Prioritization
}

\author{Gennaro Notomista$^1$, Siddharth Mayya$^1$, Mario Selvaggio$^2$, Mar\'ia Santos$^1$, and Cristian Secchi$^3$%
\thanks{\copyright~2020 IEEE. Paper accepted for publication at the 2020 IEEE International Conference on Robotics and Automation. Personal use of this material is permitted.  Permission from IEEE must be obtained for all other uses, in any current or future media, including reprinting/republishing this material for advertising or promotional purposes, creating new collective works, for resale or redistribution to servers or lists, or reuse of any copyrighted component of this work in other works.}
\thanks{$^1$G. Notomista, S. Mayya, and M. Santos are with the Institute for Robotics and Intelligent Machines, Georgia Institute of Technology, Atlanta, Georgia, USA {\tt\small \{g.notomista, siddharth.mayya, maria.santos\}@gatech.edu}}%
\thanks{$^2$M. Selvaggio is with the Department of Electrical Engineering and Information Technology, University of Naples Federico II, Naples, Italy {\tt\small mario.selvaggio@unina.it}}
\thanks{$^3$C. Secchi is with the Department of Science and Methods of Engineering, University of Modena and Reggio Emilia, Modena, Italy {\tt\small cristian.secchi@unimore.it}}
}

\begin{document}
\maketitle
\thispagestyle{empty}
\pagestyle{empty}

\begin{abstract}
Executing multiple tasks concurrently is important in many robotic applications. Moreover, the prioritization of tasks is essential in applications where safety-critical tasks need to precede application-related objectives, in order to protect both the robot from its surroundings and vice versa. Furthermore, the possibility of switching the priority of tasks during their execution gives the robotic system the flexibility of changing its objectives over time. In this paper, we present an optimization-based task execution and prioritization framework that lends itself to the case of time-varying priorities as well as variable number of tasks. We introduce the concept of \textit{extended set-based tasks}, encode them using control barrier functions, and execute them by means of a constrained-optimization problem, which can be efficiently solved in an online fashion. Finally, we show the application of the proposed approach to the case of a redundant robotic manipulator.
\end{abstract}

\section{INTRODUCTION} \label{sec:intro}

The ability of executing multiple tasks simultaneously constitutes an essential aspect of many robotic systems. Indeed, it becomes necessary in all those cases where, besides the accomplishment of application-related goals, safety requirements have to be fulfilled.
The concept of \textit{redundancy} (see, e.g.,~\cite{siciliano2010robotics}) is what makes the concurrent execution of a set of tasks feasible. If a robot is redundant with respect to a task, there exist multiple configurations of the robot that allow it to achieve that task. Within a set of tasks that are to be executed simultaneously, it is generally desirable to establish a \textit{prioritized stack}. For instance, if tasks are safety-critical---such as avoiding joint limits of robotic manipulators, or obstacles in the case of mobile robots---they must take precedence over the execution of other objectives---e.g., visual servoing or trajectory tracking~\cite{di2018safety}.

In this paper, we present an optimization-based control framework that allows robots to execute and prioritize a set of tasks. Moreover, this formulation lends itself to be applied to the case of time-varying task priorities as well as variable number of tasks to be executed. As a matter of fact, the prioritization order within the stack of tasks need not be unyielding. The importance of some tasks over others may evolve over time to reflect endogenous changes of the system---related, for instance, to changes in the application objective---or exogenous factors---such as human teleoperators in shared-control applications~\cite{selvaggio2019passive}. Due to the fact that the discrete nature of priority switches may induce discontinuities in the robot controller (see, e.g.,~\cite{Keith2011}), the problem of time-varying priorities has been addressed in many multi-task execution frameworks, as will be discussed more in detail in the next section. Owing to its formulation, the approach presented in this paper intrinsically allows task priorities to be swapped, and tasks to be inserted and removed, in a continuous fashion.

By extending the definition of set-based tasks~\cite{moe2016set}, we propose an optimization-based approach that encodes tasks by means of control barrier functions~\cite{ames2019control}. The latter have been introduced in their modern formulation in~\cite{ames2014control} for ensuring safety of dynamical systems, intended as the forward invariance property of a subset of the state space of the system. Additionally, in~\cite{xu2015robustness}, it has been shown that they can be also employed to achieve set stability. In the approach proposed in this paper, we leverage these two properties, and extend them to the time-varying case for input-output dynamical systems, in order to encode a large variety of robotic tasks as \textit{extended set-based} tasks. Moreover, under mild assumptions, we prove the continuity of the robot controller required to execute a variable number of tasks whose priorities change over time. Finally, even though the proposed framework applies to multiple kinds of robotic systems, in this paper we focus on the case of robotic manipulators (see, e.g.,~\cite{ames2013towards} for a similar approach used for humanoid robots, or~\cite{notomista2019optimal} for an example of application to the multi-robot domain).

\section{Literature Review}
\label{sec:relatedWork}

Specifying the behavior of robots in the task space is generally simpler and more intuitive than doing it in the input space. Consequently, building a controller in the task space and mapping the control input into the joint space using the (pseudo-)inverse of the Jacobian is very convenient (see e.g.,~\cite{khatib1987,siciliano1991}). This kind of approach has been extended for dealing with humanoids~\cite{mansard2009} and multi-robot systems~\cite{antonelli2009}. The implementation of the main task usually involves  fewer degree-of-freedom (DoF) than the ones available. Thus, it is possible to exploit the so-called \textit{task redundancy} of the system in order to implement extra, secondary, tasks. This can be done by arranging the secondary tasks into a hierarchical stack of tasks and to exploit the Jacobian null space projection for finding the joint velocities that implement the secondary tasks without interfering with the main task, as shown in~\cite{siciliano1991,antonelli2009}. 

When the hierarchy of tasks to be executed is dynamic, it is important to allow a smooth, stable and efficient transition  strategy between the tasks to be implemented and, consequently, on the controls provided to the robot. A lot of work has been done in order to include this feature in the operational space approach. As shown in~\cite{Keith2011}, for instance, task swapping leads to a discontinuity in the control action. Consequently, effective strategies for smoothing the control action during task swapping have been proposed  in~\cite{Keith2011,Jarquin2013}, but their reactivity is limited. As shown in~\cite{Lee2012} and~\cite{Flacco2015}, it is possible to achieve control continuity by scaling or modifying the tasks to be achieved. In~\cite{Liu2015}, \cite{Dehio2018} and \cite{Dehio2019}, modified projectors are introduced for implementing a soft prioritization that allows to smoothly rearrange the, possibly relaxed, tasks.

In order to handle dynamic task priorities, in \cite{kanoun2011kinematic} a prioritized task regulation framework based on a sequence of quadratic programs has been proposed. Hierarchical quadratic programming has also been exploited in~\cite{escande2014hierarchical,kim2019continuous}. The main advantage of recasting the task regulation into a sequence or quadratic programs is the higher freedom in specifying the tasks with respect to the operational framework. Moreover, the approach affords swapping the tasks in real time at the cost of solving a number of increasingly complex quadratic programs, which can be challenging for real-time execution.
Jacobian-based task execution and prioritization, as well as the hierarchical quadratic programming approach, shares multiple features with the method presented in this paper. In the next section, we show how Jacobian-based tasks can be encoded using our notion of extended set-based tasks. Moreover, the framework we propose encodes tasks using constraints of an optimization-based controller---just like in the hierarchical quadratic programming case---but with the advantage of solving a single optimization problem. This is achieved by leveraging control barrier functions, originally formulated to ensure a safety property through the selection of a proper control input that can be found by solving a simple convex optimization problem \cite{ames2019control}. Control barrier functions have been successfully exploited for controlling the behavior of robots in multiple domains (see, e.g.,~\cite{notomista2018,wang2018,guerrero2020,lindemann2019,talignanilandi2019}), and they naturally allow the execution of multiple tasks~\cite{wang2016multi,glotfelter2018}.

\section{TASK EXECUTION AND PRIORITIZATION}
\label{sec:main}

The objective of this section is to develop the proposed task-execution and task-prioritization framework and to illustrate its ability to synthesise continuous-controllers even when task priorities are time-varying.

Although the proposed formulation can be applied to other types of robotic systems (see, e.g., \cite{notomista2019optimal} or \cite{notomista2018constraint}), in this paper, we will focus on $n$-DoF robotic manipulators. 
For these types of robots, the so-called \textit{Jacobian-based} tasks are generally described by the following expression:
\begin{equation}
\dot \sigma = J(q) \dot q,
\end{equation}
where $q\in\R^n$ is the vector of generalized coordinates representing the state of the robot, $\sigma(q)\in\R^m$ is a function of the robot state, and $J(q) \in \R^{m\times n}$ is its Jacobian. 
In its basic meaning, a Jacobian-based task is said to be accomplished when the value of $\sigma(q)$ is regulated to a desired value $\sigma_0(q)$. 

\subsection{Extended Set-Based Tasks}
\label{subsec:esbt}

Set-based tasks have been introduced in \cite{escande2014hierarchical} as tasks characterized by a desired \textit{area of satisfaction}, which corresponds to the task being executed.
Typical tasks which are also examples of set-based tasks consist in avoidance of joint limits or obstacles, in which the areas of satisfaction are represented by the intervals of allowed joint angles and the collision-free space, respectively. 
In this definition of tasks, we notice an analogy with the concept of \textit{forward invariance} (also referred to as \textit{safety}) in the dynamical system literature \cite{khalil2015nonlinear}. 
A set-based task can be seen as keeping a desired set forward invariant, i.e., ensuring that the state of the robot never leaves the set. 
In the following definition, we propose an extension of this concept of set-based tasks in order to include the possibility of executing a task by going towards a set as well---which, in the dynamical system literature, corresponds to the well-known concept of set \textit{stability}.

\begin{definition}[Extended Set-Based Task]\label{def:esbt}
An \textit{extended set-based task} is a task characterized by a set $\mc C \subset \mc T$, where $\mc T$ is the task space, which can be expressed as the zero superlevel set of a continuously differentiable function $h \colon \mc T \times \R_{\ge0} \to \R$ as follows:
\begin{equation}
\label{eq:defSafeset}
\mc C = \{ \sigma \in \mc T \colon h(\sigma,t)\ge0 \}.
\end{equation}
The goal is rendering the time-varying set forward invariant and asymptotically stable.
\end{definition}
The following example ties the concept of extended set-based tasks to the one of Jacobian-based tasks.

\begin{example}[Generalization of Jacobian-based tasks---Part I]
\label{exmp:partI}

Consider the Jacobian-based task encoded by the following differential kinematic equation:
\begin{equation}
\label{eq:exampleJbt}
\dot \sigma_0 = J_0(q) \dot q,
\end{equation}
where $\dot \sigma_0$ is a desired function we want the robot to track. 
Let the function $h$ used in Definition~\ref{def:esbt} be defined as
\begin{equation}
\label{eq:exampleH}
h(\sigma_0,t) = -\frac{1}{2}\|\sigma-\sigma_0(t)\|^2,
\end{equation}
where the function $\sigma_0(t)$ is given by
%
$\sigma_0(t) = \int_{0}^{t} \dot\sigma_0(\tau)\de \tau$.
%
Assuming that the $\dot \sigma_0$ specified in \eqref{eq:exampleJbt} is a continuous function of time, $h$ is continuously differentiable with respect to both its arguments, as requested by the Definition~\ref{def:esbt}.

We notice that the set $\mc C$, zero superlevel of $h$, is given by $\{ \sigma\in\mc T\colon\sigma=\sigma_0(t) \}$. Therefore, rendering $\mc C$ asymptotically stable and forward invariant corresponds to the original Jacobian-based task being accomplished. In Example~\ref{exmp:partII} in Section~\ref{subsec:execution}, a method is proposed for the execution of this task using control barrier functions.
\end{example}

\subsection{Extended Set-Based Task Execution}
\label{subsec:execution}

In this section, we propose a framework to execute the extended set-based tasks defined above which will ultimately allow for time-varying prioritization among tasks in Section \ref{subsec:prioritization}.
Throughout this paper, we will suppose we can model the robot with a control affine dynamical system. 
Moreover, we will assume we have access to an output variable, $\sigma\in\mc T$, which represents the task variable. 
Therefore, the model of the robot is given by
\begin{equation} \label{eq:ca-dyn}
\begin{cases}
\dot x = f(x) + g(x)u\\
\sigma = k(x),
\end{cases}
\end{equation}
where $x\in\mc X\subseteq\mathbb{R}^n$, $u \in \mc U \subseteq \mathbb{R}^p$, $\sigma\in\mc T$, $f$ and $g$ are Lipschitz-continuous vector fields, and $k\colon\R^n\to\mc T$ is a smooth map representing the task forward kinematics.

The model~\eqref{eq:ca-dyn} encompasses both dynamic and kinematic models of robotic manipulators. 
In case the nonlinear second-order dynamic model of a robot is considered (see, e.g., \cite{siciliano2010robotics}),
by defining the joint angles and velocities as the state $x=[x_1,x_2]\tr=[q,\dot q]\tr$ and the joint torques as the input $u$, 
the system can be brought to control affine form.
If, instead, a velocity-resolved robot control scheme is considered~\cite{siciliano2010robotics}, then the system follows the single-integrator dynamics $\dot x = u$---which is control affine too---, where the state $x$ is the vector of joint coordinates and the input $u$ is the vector of joint velocities. 
Without loss of generality, in the following we will focus on the kinematic model of robotic manipulator, assuming that we can control its joint velocities.

Based on a previously developed method for executing tasks~\cite{notomista2018constraint}, we now present an optimization-based framework required to execute extended set-based tasks. 
As described in Section~\ref{sec:intro}, we consider tasks whose execution can be encoded as the minimization of a continuously-differentiable cost function $C \colon \mc T \times \mathbb{R}_{\geq 0} \to \mathbb{R}$.
This can be written as the following optimization problem:
\begin{equation} \label{eq:taskmin}
\begin{aligned}
\minimize_u &C(\sigma,t)\\
\st & \dot x = f(x) + g(x)u\\
& \sigma = k(x).
\end{aligned}
\end{equation}
Following the approach proposed in~\cite{notomista2018constraint} and~\cite{notomista2019optimal}, this problem can be reformulated as a constraint-based optimization problem where the robot minimizes its control input subject to a constraint which enforces the execution of the task itself. 
These constraints are enforced using Control Barrier Functions (CBFs)~(see \cite{ames2019control} for a comprehensive overview on the subject).  CBFs perfectly lend themselves to encode extended set-based tasks as they are able to ensure both set safety and set stability. These properties will be recalled in the following for the notion of CBFs extended to encompass constraints on the output of a dynamical system.

\begin{definition}[Output Time-Varying Control Barrier Function, based on~\cite{ames2019control} and~\cite{notomista2019persistification}]
\label{def:cbf}
Let $\mc C \subset \mc D \subset \mc T$ be the superlevel set of a continuously differentiable function $h: \mc D \times \R_{\ge0} \to \R$, contained in a domain $\mc D$. 
Then, $h$ is an output time-varying control barrier function---in the following referred to  simply as CBF---if there exists a Lipschitz continuous extended class $\mc K_\infty$ function (\cite{ames2019control}) $\gamma$ such that for the control system~\eqref{eq:ca-dyn}, for all $\sigma \in \mc D$,
\begin{align}
\label{eq:cbfDefinition}
\sup_{u \in \mc U}  \left\{ \frac{\partial h}{\partial t} + \frac{\partial h}{\partial \sigma} \frac{\partial \sigma}{\partial x} f(x) + \frac{\partial h}{\partial \sigma} \frac{\partial \sigma}{\partial x} g(x) u \right\} \geq - \gamma(h(\sigma)).\hspace{0.25cm}
\end{align}
\end{definition}

With this definition of CBF, we can now state the main theorem whose results will be used in the following sections to derive the multi-task execution and prioritization framework.

\begin{theorem}[based on \cite{ames2019control} and \cite{notomista2019persistification}]
\label{thm:cbf}
Let $\mc C \subset \mc T$ be a set defined as the superlevel set of a continuously differentiable function $h: \mc D \times \R_{\ge0} \to \R$.  
If $h$ is a CBF on $\mc D$ according to Definition~\ref{def:cbf}, then any Lipschitz continuous controller
$u(x,t)\in\mc V(x,t)$, where $\mc V(x,t) = \{ u(x,t)\colon \frac{\partial h}{\partial t} + \frac{\partial h}{\partial \sigma} \frac{\partial \sigma}{\partial x} f(x) + \frac{\partial h}{\partial \sigma} \frac{\partial \sigma}{\partial x} g(x) u + \gamma(h(\sigma)) \ge0 \}$, for the system~\eqref{eq:ca-dyn}, renders the set $\mc C$ forward invariant.  
Additionally, the set $\mc C$ is asymptotically stable in $\mc D$.
\end{theorem}

\begin{figure}
\centering
\includegraphics[width=0.48\textwidth]{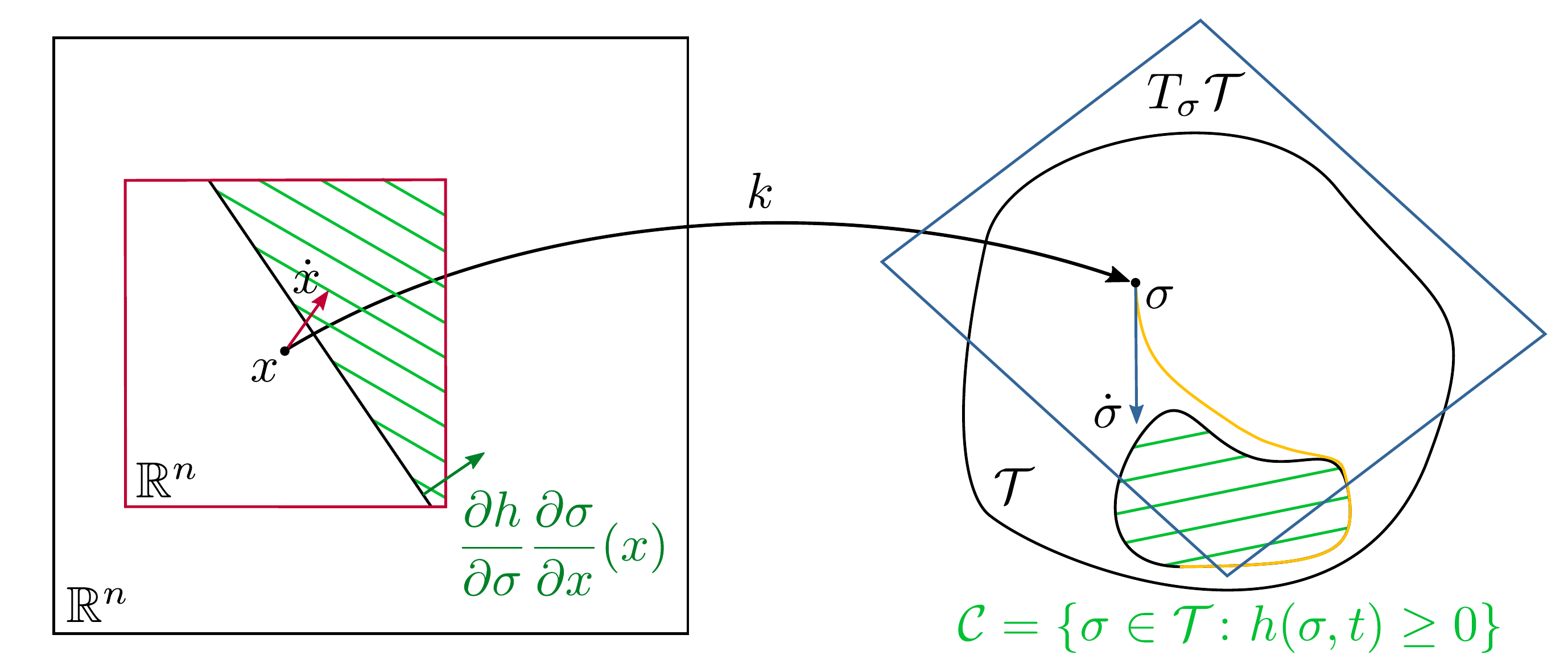}
\caption{Pictorial representation of the differential constraints introduced in \eqref{eq:cbfDefinition}, for a time-invariant CBF. The set $\mc C$ to be rendered forward invariant and asymptotically stable is a subset of the task space $\mc T$. The constraint \eqref{eq:cbfDefinition} can be equivalently written as $\frac{\partial h}{\partial\sigma} \dot\sigma\ge-\gamma(h(\sigma,t))$, which is a affine in $\dot\sigma$, constraining it to lie in a half plane in the tangent space $T_\sigma \mc T$ of the task space. The velocity $\dot\sigma$ is transformed into the velocity $\dot x$ in the tangent space $T_x\R^n\cong\R^n$ of the state space. In this space, the constraint \eqref{eq:cbfDefinition} is affine in $\dot x$, which has to lie in the green hatched half plane.}
\label{fig:cbf-depict}
\end{figure}
In the case of robotic manipulators, Definition~\ref{def:cbf} and Theorem~\ref{thm:cbf} describe how CBFs ensure the forward invariance and the asymptotic stability of a subset of the task space $\mc T$ by enforcing constraints on the joint velocities in the tangent space of the state space. This is depicted in Fig.~\ref{fig:cbf-depict}, where the set to be rendered asymptotically stable is marked by a green hatching. The velocity $\dot\sigma$ is transformed to the velocity $\dot x$ by means of the Jacobian $J(q)$ of the transformation $k$. Then, the inequality in \eqref{eq:cbfDefinition} corresponds to constraining the vector $\dot x$ to lie in the half plane of $\R^n$ hatched in green.

Proceeding similarly to \cite{notomista2018constraint}, it can be shown how a control input $u(t)$ generated using \eqref{eq:taskmin} is equivalent to solving the following constrained optimization problem:
\begin{equation}
\begin{aligned} \label{eqn_const_opt} 
\minimize_{u} &\|u\|^2 \\
\st & \frac{\partial h}{\partial t} + \frac{\partial h}{\partial \sigma} \frac{\partial \sigma}{\partial x} f(x) + \frac{\partial h}{\partial \sigma} \frac{\partial \sigma}{\partial x} g(x) u \\
&+ \gamma(h(\sigma,t)) \ge0,
\end{aligned}
\end{equation}
where $h(\sigma,t) = -C(\sigma,t)$ is a CBF. See~\cite{notomista2018constraint} for a detailed discussion on this equivalence. \par 

\begin{example}[Generalization of Jacobian-based tasks---Part II]
\label{exmp:partII}

The execution of the task introduced in Example~\ref{exmp:partI} can be done by ensuring the asymptotic stability of the set $\mc C$ defined in~\eqref{eq:defSafeset}. 
This property, in turn, can be achieved by treating $h$ as a CBF. 
As a result, by the definition of the set $\mc C$ in~\eqref{eq:defSafeset} and the function $h$ in~\eqref{eq:exampleH}, ensuring the asymptotic stability of the set $\mc C$ corresponds to the following condition: $\sigma(t) \to \sigma_0(t)$, as $t\to\infty$.

We can now use the optimization-based formulation~\eqref{eqn_const_opt}, in order to synthesize a velocity controller $\dot q$ that is able to render the set $\mc C$ asymptotically stable as well as forward invariant. 
The controller solution of
\begin{equation}
\label{eq:exampleQp}
\begin{aligned}
\minimize_{\dot q} &\|\dot q\|^2\\
\st &(\sigma-\sigma_0(t))\tr J_0(q) \dot q \ge\\ &-\gamma(h(\sigma_0,t))-(\sigma-\sigma_0(t))\tr\dot\sigma_0(t)
\end{aligned}
\end{equation}
guarantees the asymptotic stability of $\mc C$, which is equivalent to the execution of the Jacobian-based task~\eqref{eq:exampleJbt}. Moreover, owing to its convexity, solving the optimization problem~\eqref{eq:exampleQp} can be done very efficiently~\cite{boyd2004convex}. 
Therefore, this approach also lends itself to applications where real-time requirements are prescribed.
\end{example}

In the next section, we demonstrate how the formulation in~\eqref{eqn_const_opt} can be extended to allow executing and prioritizing multiple tasks.

\subsection{Extended Set-Based Multi-Task Prioritization}
\label{subsec:prioritization}

Consider a set of tasks, denoted as $T_1,\ldots,T_M$, that need to be executed and are encoded by cost functions $C_1,\ldots,C_M$, respectively. 
The execution of these tasks can be achieved by solving:
\begin{equation}
\begin{aligned} \label{eqn_const_opt_multiple} 
\minimize_{u} &\|u\|^2 \\
\st & \frac{\partial h_m}{\partial t} + \frac{\partial h_m}{\partial \sigma} \frac{\partial \sigma}{\partial x} f(x) + \frac{\partial h_m}{\partial \sigma} \frac{\partial \sigma}{\partial x} g(x) u \\
&+ \gamma(h_m(\sigma,t)) \ge0\quad\forall m \in \{1,\ldots,M\},
\end{aligned}
\end{equation}
where $h_m(\sigma,t) = -C_m(\sigma,t)$. 
While this formulation provides a convenient way to compose multiple tasks together, the solution to such an optimization problem is not guaranteed to exist due to the incompatibility of conflicting task requirements. 
We now modify this formulation and introduce a slack variable $\delta_m$ corresponding to each task, which denotes the extent to which the constraint corresponding to task $T_m$ can be relaxed. 
Thus, for the set of $M$ tasks $T_1,\ldots,T_M$, \eqref{eqn_const_opt_multiple} is now modified as: 
\begin{equation}
\begin{aligned} \label{eq:udeltaqp} 
\minimize_{u,\delta} &\|u\|^2 + l\|\delta\|^2 \\
\st & \frac{\partial h_m}{\partial t} + \frac{\partial h_m}{\partial \sigma} \frac{\partial \sigma}{\partial x} f(x) + \frac{\partial h_m}{\partial \sigma} \frac{\partial \sigma}{\partial x} g(x) u \\
&+ \gamma(h_m(\sigma,t)) \ge - \delta_{m}\quad\forall m \in \{1,\ldots,M\},
\end{aligned}
\end{equation}
where $\delta = [\delta_1,\ldots,\delta_M]^T$ denotes the slack variables corresponding to each task executed by the robot, and $l~\ge~0$ is a scaling constant.
Within this framework, a natural way to introduce priorities is to enforce relative constraints among the slack variables corresponding to the different tasks that need to be performed. 
For example, if the robot were to perform task $T_m$ with the highest priority (often referred to using the partial order notation $T_m\prec T_n$ $\forall n \in \{1,\ldots,M\}, n\neq m$), the additional constraint $\delta_m \leq \delta_n/\kappa$, $\kappa>1$, $\forall n \in \{1,\ldots,M\}, n\neq m$ in \eqref{eq:udeltaqp} would imply that task $T_m$ is relaxed to a lesser extent---thus performed with a higher priority---than the other tasks. The relative scale between the functions $h_m$ encoding the tasks should be considered while choosing the value of $\kappa$.
In general, such prioritizations among tasks can be encoded through an additional linear constraint $K\delta \geq 0$ to be enforced in the optimization problem~\eqref{eq:udeltaqp}. In the following, we refer to $K$ as the \textit{prioritization matrix}, which encodes the pairwise inequality constraints among the slack variables, thus fully specifying the prioritization stack among the tasks.

\subsection{Examples of Applications} \label{subsec:examples}
We now consider three examples in order to highlight how the proposed framework can effectively enable the prioritization and execution of multiple safety- as well as non-safety-critical tasks. Each scenario presented below considers three tasks, some of which are safety-critical---i.e., they always need to be executed, regardless of the prioritization stack. The three examples share the formulation in \eqref{eq:udeltaqp} with the additional constraint $K\delta\ge0$.

\subsubsection{3 Non-Safety-Critical Prioritized Tasks}

The 3 non-safety-critical tasks are encoded by the CBFs
\begin{equation}
\label{eq:exmpTask}
h_m:\mc T_m\times\R_{\ge0}\to\R, \quad\quad m\in\{1,2,3\},
\end{equation}
where $\mc T_m$ is the task space related to task $m$. The values of the entries of the prioritization matrix $K$ have to be chosen in order to encode the desired priorities between the tasks. For instance, if the stack of tasks is $T_1\prec T_3\prec T_2$, then the rows of $K$ have to encode the constraints $\delta_1\le\delta_3/\kappa$ and $\delta_3\le\delta_2/\kappa$ (as described in Sect.~\ref{subsec:prioritization}), i.e.,
\begin{equation}
K = \begin{bmatrix}
-1&0&1/\kappa\\
0&1/\kappa&-1
\end{bmatrix}.
\end{equation}

\subsubsection{1 Safety-Critical Task and 2 Non-Safety-Critical Prioritized Tasks}
\label{subsec:examples12}

Assume task $T_1$ is safety-critical and tasks $T_2$ and $T_3$ are not, where each $T_m$ is encoded as in \eqref{eq:exmpTask}. In this case, $\delta_1\overset{!}{=}0$ and the prioritization matrix $K$ is given by
\begin{equation}
K(t) = \begin{cases}
\begin{bmatrix}0&-1&1/\kappa\end{bmatrix} &\text{ if }T_2\prec T_3,\\[0.25cm]
\begin{bmatrix}0&1/\kappa&-1\end{bmatrix} &\text{ if }T_2\succ T_3.
\end{cases}
\end{equation}
Notice that, as task $T_1$ is safety-critical, it will always be executed as its corresponding slack variable is set to 0, whereas tasks $T_2$ and $T_3$ will be executed only if their slack variables can be driven to 0. This condition is analogous to the notions of task \textit{independence} and \textit{orthogonality} (\cite{antonelli2009tro}).

\subsubsection{2 Safety-Critical Tasks and 1 Non-Safety-Critical Task}

Now, assume that tasks $T_1$ and $T_2$ are safety-critical and task $T_3$ is not. In this case, there is no need for prioritizing. In fact, by definition of safety-critical tasks, $T_1$ and $T_2$ must be always executed ($\delta_1=\delta_2\overset{!}{=}0$). Task $T_3$ will be executed only if it is compatible with the execution of the first two (see discussion in Sect.~\ref{subsec:examples12}).
Note further that, since the safety-critical tasks $T_1$ and $T_2$ do not have to be prioritized, they could even be combined in a single task using the techniques developed in \cite{wang2016multi} or \cite{glotfelter2018}.

These examples demonstrate how our framework can encode the prioritization of different task stacks, especially when multiple tasks might be safety-critical. As the objective of this paper is to illustrate how such a formulation can also allow for dynamically evolving task prioritizations, the next section considers a time-varying prioritization matrix $K(t)$, and demonstrates how this lends itself to the synthesis of continuous controllers for task switching and task insertion/removal.

\subsection{Time-Varying Priorities} \label{subsec:switching}

As the need of switching priorities between tasks in an online fashion arises in multiple applications (see discussions in Section~\ref{sec:intro} and \ref{sec:relatedWork}), in this section, the application of the control framework proposed in this paper to the case of dynamic priorities is shown. 
The main problem when priorities are switched is that, due to the intrinsically discrete nature of the order among tasks, discontinuities in the controller might arise during the switching transient \cite{Keith2011}.
As discussed in Section~\ref{sec:relatedWork}, several methods have been proposed to mitigate this effect, which are tailored to the specific task execution or prioritization approach they target. \par 
In the following proposition, we give sufficient conditions to ensure the continuity of the robot control input during the priority reordering of a stack of extended set-based tasks executed using the optimization-based control framework presented in this paper.
\begin{proposition}[Continuity of the controller during task priority switching]
\label{prop:switch}
Consider the following optimization problem:
\begin{equation}
\label{eq:qpSwitch}
\begin{aligned}
&[u^*(t), \delta^*(t)] =\\[0.1cm]
&\begin{aligned}
\argmin_{u,\delta} ~&\|u\|^2 + l\|\delta\|^2 \\
\st & \frac{\partial h_m}{\partial t} + \frac{\partial h_m}{\partial \sigma} \frac{\partial \sigma}{\partial x} f(x) + \frac{\partial h_m}{\partial \sigma} \frac{\partial \sigma}{\partial x} g(x) u \\
&+ \gamma(h_m(\sigma,t)) \ge - \delta_{m}\quad\forall m \in \{1,\ldots,M\}\\
&K(t)\delta\ge0,
\end{aligned}
\end{aligned}
\end{equation}
where $l>0$, $h_m=-C_m$, $m=1,\ldots,M$ encode $M$ tasks through the continuously differentiable cost functions $C_m$, and $K\colon\R_{\ge0}\to\R^{N_p\times M}$ is a mapping that, for each time instant $t$, provides a $N_p\times M$ prioritization matrix, $N_p$ being the number of constraints required to characterize the desired task prioritizations\footnote{Notice that $N_p\le M^2$ as any other pairwise constraint would be redundant.}.
If $K$ is Lipschitz continuous in time, then the controller $u^*(t)$, solution of \eqref{eq:qpSwitch}, is Lipschitz continuous in time.
\end{proposition}
\begin{proof}
As discussed in Section~\ref{subsec:execution} and in~\cite{notomista2019optimal}, thanks to the presence of the slack variables $\delta_m$ and the absence of additional constraints on $u$, the optimization problem~\eqref{eq:qpSwitch} is always feasible.
Moreover, by the assumption of $l>0$, the objective function is strictly convex. Furthermore, by assumption, all constraints in~\eqref{eq:qpSwitch} are Lipschitz continuous with respect to time.
Then, Theorem~1 in~\cite{morris2013sufficient} ensures that the solution of~\eqref{eq:qpSwitch} is Lipschitz continuous.
\end{proof}

A limiting case of priority switching consists in the insertion or the removal of a task from a given stack of tasks. 
Using the optimization-based control framework proposed in this paper, the following proposition gives sufficient conditions required in order to perform such a maneuver while ensuring the continuity of the velocity controller of the robot.
\begin{proposition}[Continuity of the controller during task insertion]
\label{prop:insRm}
Consider the following optimization problem:
\begin{equation}
\label{eq:qpInsRm}
\begin{aligned}
&[u^*(t), \delta^*(t)] =\\[0.1cm]
&\begin{aligned}
\argmin_{u,\delta} ~&\|u\|^2 + l\|\delta\|^2\\
\st & \frac{\partial h_m}{\partial t} + \frac{\partial h_m}{\partial \sigma} \frac{\partial \sigma}{\partial x} f(x) + \frac{\partial h_m}{\partial \sigma} \frac{\partial \sigma}{\partial x} g(x) u \\
&+ \gamma(h_m(\sigma,t)) \ge - \delta_{m}\quad\forall m \in \{1,\ldots,M-1\}\\
&\frac{\partial h_M}{\partial \sigma} \frac{\partial \sigma}{\partial x} g(x) u +  \delta_M \\
&\ge \bigg(\!\!-\frac{\partial h_M}{\partial t} - \frac{\partial h_M}{\partial \sigma} \frac{\partial \sigma}{\partial x} f(x)+ \gamma(h_M(\sigma,t))\bigg)\rho(t)\\
&K(t)\delta\ge0,
\end{aligned}
\end{aligned}
\end{equation}
where $l\ge0$, $\rho\colon\R_{\ge0}\to\R_{\ge0}$, and $K\colon\R_{\ge0}\to\R^{N_p\times M}$ is a mapping defined as in Proposition~\ref{prop:switch}.

Let $t_\text{ins}\ge0$ be the time at which task $T_M$ needs to be inserted. If $K$ and $\rho$ are continuous functions, with $\rho\equiv0$ on $(-\infty,t_\text{ins}]$ and $\rho\equiv1$ on $[t_\text{ins}+\Delta t_\text{ins},\infty)$, then the controller $u^*$, solution of~\eqref{eq:qpInsRm}, is a continuous function of time and allows task $T_M$ to be inserted at time $t_\text{ins}+\Delta t_\text{ins}$ in the stack of tasks with any desired priority.
\end{proposition}
\begin{proof}
For $t\le t_\text{ins}$, as $\rho(t)\equiv0$, the constraint
\begin{equation}
\label{eq:taskM}
\begin{aligned}
&\frac{\partial h_M}{\partial \sigma} \frac{\partial \sigma}{\partial x} g(x) u +  \delta_M\\
\ge &\bigg(-\frac{\partial h_M}{\partial t} - \frac{\partial h_M}{\partial \sigma} \frac{\partial \sigma}{\partial x} f(x)+ \gamma(h_M(\sigma,t))\bigg)\rho(t)
\end{aligned}
\end{equation}
related to task $T_M$ is not active. In fact, $(u,\delta)=(0,0)$ satisfies \eqref{eq:taskM} and would achieve a (global) minimum of the cost function. For $t>t_\text{ins}$, the constraint \eqref{eq:taskM} is effectively inserted in the optimization problem. As, by hypotheses, $\rho(t)$ is continuous and $\rho(t)=0$ for $t\le t_\text{ins}$, $\rho(t)\to0$ as $t\to t_\text{ins}$. Therefore, the right hand side of \eqref{eq:taskM} also converges to 0 as $t\to t_\text{ins}$. Thus, by Corollary 3.1 in \cite{lee2006continuity}, the solution $u^*(t)$ is continuous at $t=t_\text{ins}$. Continuity for all $t>t_\text{ins}$ stems from the continuity of $\rho$ and of all the other parameters of the optimization problem. For $t\ge t_\text{ins}+\Delta t_\text{ins}$, the condition $\rho\equiv~1$ allows task $T_M$ to be executed together with the other tasks with priorities encoded by the prioritization matrix $K(t)$. Hence, the controller $u^*(t)$, solution of \eqref{eq:qpInsRm}, is a continuous velocity controller for the robot, that is able to insert task $T_M$ into the stack of tasks with any desired priority.
\end{proof}

\begin{remark}
A result analogous to Proposition~\ref{eq:qpInsRm} can be stated to show the continuity of the robot controller when a task is removed from the stack of prioritized tasks, rather than inserted. If task $T_M$ is to be removed at the time $t_\text{rem}$, $\rho(t)$, besides being continuous, has to go to 0 as $t\to t_\text{rem}$.
\end{remark}

In the next section, we present the results of a simulated experiment performed using a manipulator arm. The experiment includes insertion of safety- as well as non-safety-critical tasks in a time-varying prioritized stack of tasks.

\section{EXPERIMENTS}
\label{sec:experiments}
\begin{figure}
	\centering
	\includegraphics[width=0.35\textwidth]{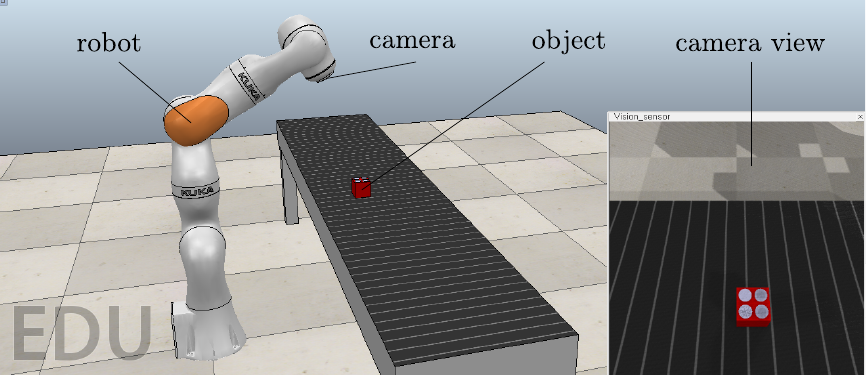}
	\caption{Experimental setup: a 7-DoF manipulator performing multiple tasks in an simulated industrial environment~\cite{Freese2013}.
	A camera mounted on the end-effector 
	is used to perform visual servoing tasks.}
	\label{fig:exp_setup}
	\vspace{-10pt}
\end{figure}

The experimental setup is shown in Fig.~\ref{fig:exp_setup}. 
A 7-DoF arm is employed to perform joint control, position control of the robot end-effector, as well as visual servoing tasks. These include both safety- and non-safety-critical tasks, as described in detail in the following.
Therefore, the considered task spaces are the joint space, the Cartesian space, and the projective plane.

To this end, let $q \in \R^7$ denote the vector of the robot generalized coordinates, $p \in \R^3$ the end-effector Cartesian position, and $s \in \R^2$ the  coordinates of a feature on the image plane. 
Following the formulation in Example~\ref{exmp:partI}, these tasks can be encoded through the following CBFs: $h_T(z) = -\gamma_T\|z-z_d\|^2$.
In the expression of $h_T(z)$, $\gamma_T>0$, and $z$ is a placeholder for $q$ in the joint space task ($T_q$), $p$ in the Cartesian space task ($T_p$), and  $s$ in the visual servoing task ($T_s$). 
The value of $z_d$ denotes the desired reference value for $z$. 
Additionally, one safety-critical task is considered, namely keeping the joint variables within an upper and a lower bound.
The corresponding CBF used in this case for each joint $i$ is given by
\begin{equation}
\label{eq:cbf_joints}
    h_{T,i}\left(z_i\right) = \gamma_T \left(z^+_i - z_i\right)\left(z_i - z^-_i\right),
\end{equation}
where $z^+_i$ and $z^-_i$ denote the $i$-th upper and lower  limit, respectively, of the variable $z_i$. 
In the joint space, these bounds represent physical joint limits.
Notice how keeping the value of $z_i\in[z^-_i,z^+_i]$ corresponds to $h_T(z)\ge0$, which can be enforced using the constraint-based formulation in~\eqref{eq:qpInsRm}.

\begin{figure}[t]
	\centering
	\includegraphics[width = 0.95\linewidth]{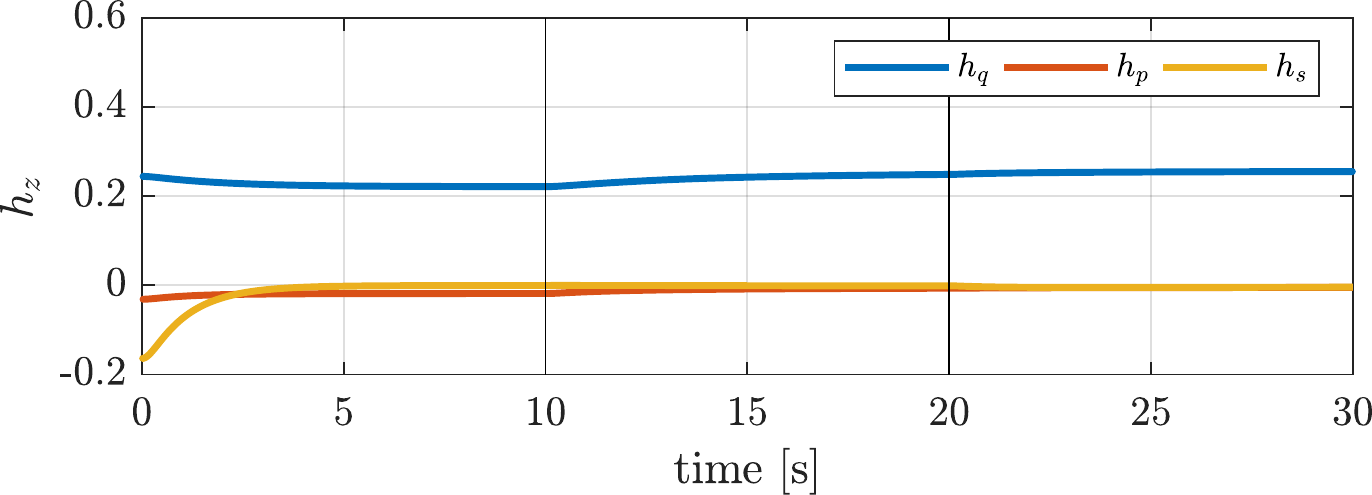}
	\includegraphics[width = 0.95\linewidth]{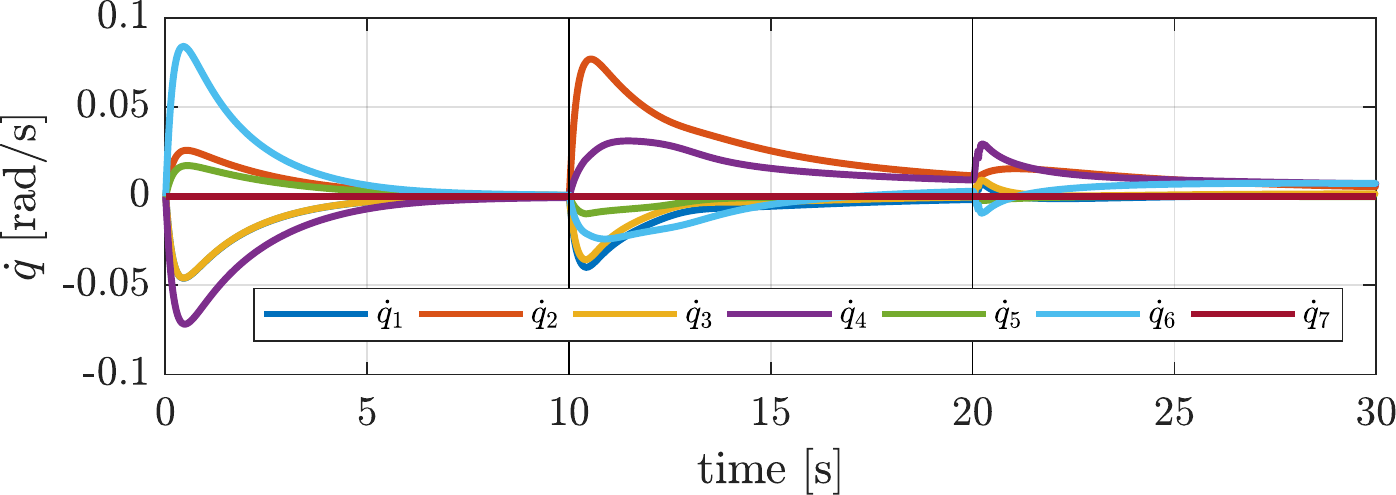}
	\caption{Graphs of the values of the CBFs of the 3 considered tasks (top), and of the joint velocity controller (bottom). 
	The vertical lines at $t=10$ s correspond to insertion of the position task, whereas the ones at $t=20$ s indicate the switch of priorities between the position control and the visual servoing tasks. In the latter case, it can be observed how the presented task prioritization framework is suitable to switch task priorities even when tasks have not been completely accomplished yet.}
	\label{fig:experimental_data}
	\vspace{-10pt}
\end{figure}

The results of a pilot experiment are shown in Fig.~\ref{fig:experimental_data}. 
At the top, the values of the CBFs corresponding to joint control ($h_q$), position control ($h_p$) and visual servoing ($h_s$) tasks are reported.
The objective is that of driving their values to 0 and/or keep them non-negative. 
As can be seen, the value of $h_q$ is always kept positive, corresponding to the condition in which joint variables do not go beyond their limits. 
The values of $h_p$ and $h_s$, instead, are driven to 0 in order to execute the prioritized tasks. 
At the bottom of Fig.~\ref{fig:experimental_data}, the joint velocity controller is plotted, showing the continuity property guaranteed by the optimization problem~\eqref{eq:qpInsRm}. 
The prioritization of the tasks during the course of the experiment is as follows. 
The safety-critical task $T_q$ always has highest priority. 
At $t=0$ s, $T_v$ is inserted: at this point, the stack of tasks is $T_q \prec T_v$. 
Then, at $t=10$ s, $T_p$ is inserted at the lowest priority level, so that the stack of tasks becomes $T_q \prec T_v \prec T_p$. 
Finally, at $t=20$ s, $T_v$ and $T_p$ are swapped, making the stack of tasks $T_q \prec T_p \prec T_v$.

Thus, the extended-based task prioritization framework presented in this paper guarantees the continuity of the robot velocity controller when applied to task insertion and priority switching scenarios. If the parameters of the optimization problem \eqref{eq:qpInsRm} (i.e., $\rho$ and $K$) vary smoothly, it is possible to guarantee the smoothness of the velocity controller. This can be used to ensure continuity of the torque controls.

\section{CONCLUSIONS}

In this paper, we presented an optimization-based framework to execute and prioritize multiple robotic tasks. We extended the definition of set-based tasks and proposed a method to execute them, which leverages control barrier functions. Time-varying priorities, as well as task insertion and removal, were handled by continuously changing the optimization problem required to synthesize the robot controller. Guarantees on continuity were provided under mild assumptions on the parameters of the optimization problem. Although the presented framework can be applied to robotic systems of different nature, its effectiveness was demonstrated through a simulated experiment involving a manipulator arm robot.

\addtolength{\textheight}{-2.5cm}

\newpage

\bibliographystyle{IEEEtran}
\bibliography{bib/nsp_cbf.bib}

\end{document}